\documentclass[letterpaper, 10 pt, conference]{ieeeconf}

\IEEEoverridecommandlockouts                              %

\title{\LARGE \bf
db-A*: Discontinuity-bounded Search for Kinodynamic Mobile Robot Motion Planning
}

\author{Wolfgang Hönig, Joaquim Ortiz-Haro, and Marc Toussaint%
\thanks{The authors are with TU Berlin, Germany. {\tt\footnotesize \{hoenig, joaquim.ortizdeharo, toussaint\}@tu-berlin.de}.}%
\thanks{The research was funded by the Deutsche Forschungsgemeinschaft (DFG, German Research Foundation) - 448549715 and by the German-Israeli Foundation for Scientific Research (GIF) grant I-1491-407.6/2019. Joaquim Ortiz-Haro thanks the International Max-Planck Research School for Intelligent
Systems (IMPRS-IS) for the support.
}%
}

\usepackage[bookmarks=true]{hyperref}
\usepackage{amssymb}
\usepackage{amsmath}
\usepackage{multirow}
\usepackage{flushend}

\newtheorem{theorem}{Theorem}
\newtheorem{definition}{Definition}
\newtheorem{remark}{Remark}
\newtheorem{example}{Example}

\usepackage[detect-weight=true, detect-family=true,detect-inline-weight=math]{siunitx}

\usepackage[ruled,vlined,linesnumbered]{algorithm2e}
\SetInd{0.25em}{0.5em}
\SetAlFnt{\footnotesize\sf}
\SetKwRepeat{Do}{do}{while}
\SetKw{KwStep}{step}
\SetKwInput{KwData}{Input}
\SetKwInput{KwResult}{Result}
\SetKwBlock{RepeatForever}{repeat}{}
\SetKw{Continue}{continue}
\SetKwComment{Comment}{$\triangleright$\ }{}
\SetCommentSty{itshape}

\usepackage[capitalise]{cleveref} %

\usepackage{tikz}
\usetikzlibrary{fit,calc}
\newcommand*{\tikzmk}[1]{\tikz[remember picture,overlay,] \node (#1) {};\ignorespaces}

\newcommand{\marklineFour}[1]{\tikz[remember picture,overlay]{\node[yshift=2pt,xshift=#1,fill=yellow!100,opacity=.25,fit={(A)($(A)+(0.95\linewidth,-3.3\baselineskip)$)},rounded corners=4pt] {};}\ignorespaces}

\newcommand{\marklineThree}[1]{\tikz[remember picture,overlay]{\node[yshift=2pt,xshift=#1,fill=yellow!100,opacity=.25,fit={(A)($(A)+(0.95\linewidth,-2.3\baselineskip)$)},rounded corners=4pt] {};}\ignorespaces}

\newcommand{\marklineOne}[1]{\tikz[remember picture,overlay]{\node[yshift=2pt,xshift=#1,fill=yellow!100,opacity=.25,fit={(A)($(A)+(0.95\linewidth,-0.3\baselineskip)$)},rounded corners=4pt] {};}\ignorespaces}

\newcommand{\vx}{\mathbf{x}}    %
\newcommand{\vu}{\mathbf{u}}    %
\newcommand{\vzero}{\mathbf{0}}    %

\newcommand{\seqX}{\mathbf{X}}    %
\newcommand{\seqU}{\mathbf{U}}    %

\newcommand{\sX}{\mathcal{X}}   %
\newcommand{\sU}{\mathcal{U}}   %
\newcommand{\sW}{\mathcal{W}}   %
\newcommand{\sM}{\mathcal{M}}   %

\newcommand{\vf}{\mathbf{f}}    %
\newcommand{\vg}{\mathbf{g}}    %
\newcommand{\vh}{\mathbf{h}}    %

\DeclareMathOperator{\step}{step}
\DeclareMathOperator*{\argmax}{argmax}

\begin{document}

\maketitle
\thispagestyle{empty}
\pagestyle{empty}

\begin{abstract}
We consider time-optimal motion planning for dynamical systems that are translation-invariant, a property that holds for many mobile robots, such as differential-drives, cars, airplanes, and multirotors.
Our key insight is that we can extend graph-search algorithms to the continuous case when used symbiotically with optimization.
For the graph search, we introduce discontinuity-bounded A* (db-A*), a generalization of the A* algorithm that uses concepts and data structures from sampling-based planners.
Db-A* reuses short trajectories, so-called motion primitives, as edges and allows a maximum user-specified discontinuity at the vertices.
These trajectories are locally repaired with trajectory optimization, which also provides new improved motion primitives.
Our novel kinodynamic motion planner, kMP-db-A*, has almost surely asymptotic optimal behavior and computes near-optimal solutions quickly.
For our empirical validation, we provide the first benchmark that compares search-, sampling-, and optimization-based time-optimal motion planning on multiple dynamical systems in different settings.
Compared to the baselines, kMP-db-A* consistently solves more problem instances, finds lower-cost initial solutions, and converges more quickly.
\end{abstract}

\section{Introduction}

Motion planning for robots with known kinodynamics remains challenging, especially when a time-optimal motion is desired.
Consider the example in \cref{fig:overview} of a simple dynamical model in 2D (unicycle, 3-dimensional state space and 2-dimensional control space).
Finding the time-optimal solution is surprisingly challenging for state-of-the-art methods when constraining the control space to model a plane with a malfunctioning rudder, i.e., with a positive minimum speed and asymmetric angular velocity limits.

Current planning approaches are sampling-based, search-based, optimization-based, or hybrid.
Each of these methods has their strengths and weaknesses.
Sampling-based planners can find initial solutions quickly and have strong guarantees for convergence to an optimal solution.
However, in practice the initial solutions are far from the optimum, the convergence rate is low, and the solutions typically require some post-processing.
Search-based approaches can remedy those shortcomings by connecting precomputed trajectories, so-called \emph{motion primitives}, using A* or related graph search algorithms.
Yet, the seemingly strong theoretical guarantees only hold up to the selected discretization of the state space and the precomputed motions.
Moreover, scaling this approach to higher dimensions has proved difficult and requires careful, frequently hand-crafted design of the motion primitives.
This curse of dimensionality can be overcome by optimization-based planners, which scale polynomially rather than exponentially with the number of state dimensions.
However, these planners are, in the general case, only locally optimal and thus require a good initial guess both for the trajectory and time horizon.

\begin{figure}
    \centering
    \includegraphics[width=\linewidth]{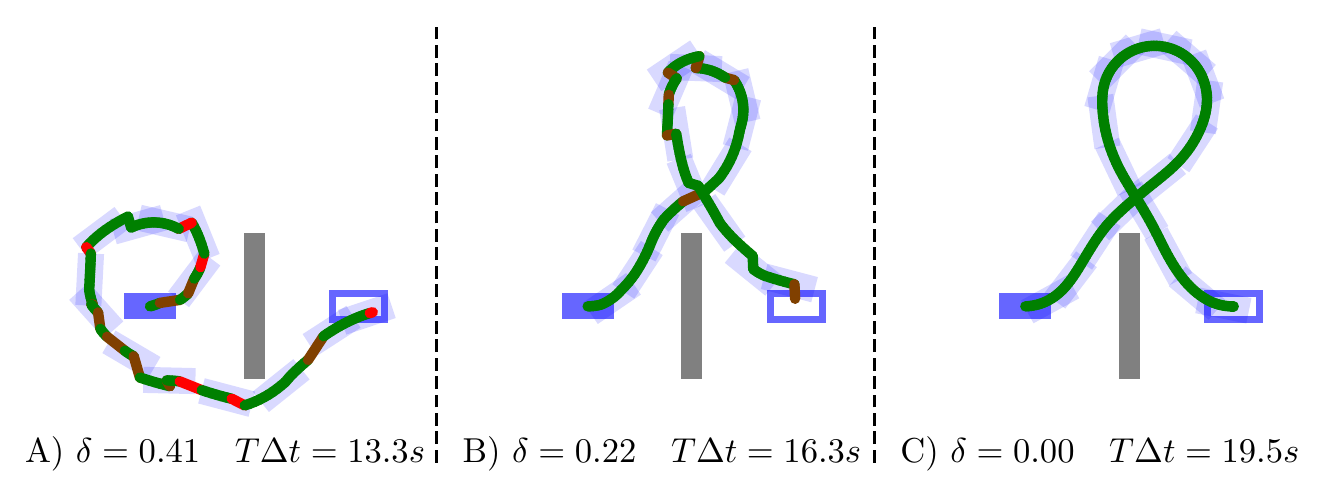}
    \caption{
    A kinodynamic planning problem, where a plane-like 2D robot with a malfunctioning rudder (no sharp right turns) has to move from the left configuration (filled blue) to the right configuration (blue outline) in minimum time.
    \emph{A)} An initial solution by ``stitching'' motion primitives with bounded discontinuities (magnitude indicated by hue of red color). \emph{B)} A refined solution using more primitives and a lower discontinuity bound $\delta$. \emph{C)} Final trajectory computed with optimization using B) as initial guess.
    }
    \label{fig:overview}
\end{figure}

In this paper, we present a new approach for kinodynamic motion planning of mobile robots that combines key ideas and strengths of the aforementioned previous methods.
We rely on a graph-search method, because it provides a theoretically grounded exploration/exploitation tradeoff, but we want to remedy its primary shortcoming of a predefined discretization, similar to sampling-based planning.
The naive approach of simply increasing the number of primitives is intractable, due to the resulting infinite number of states and infinite branching factor.
We solve this challenge with a combination of bounded-discontinuity search with nonlinear optimization.
Introducing the discontinuity makes the search tractable: we can reuse the primitives and have a finite number of states to expand.
While the resulting trajectory is not feasible, it can be used as initial guess of trajectory optimization that locally repairs the discontinuous trajectory into a valid trajectory.
We execute search and optimization in an iterative fashion, where the value of the discontinuity bound decreases in every iteration.
For large bounds, the search is very fast, but the optimizer might fail to find a valid solution.
For very small bounds, the search requires a longer runtime, but the optimizer has an excellent initial guess.
This combination results in an efficient anytime planner with probabilistic optimality guarantees.

More specifically, our first contribution is the introduction of \emph{kMP-db-A*}, a new \underline{k}inodynamic \underline{m}otion \underline{p}lanner that combines a novel search algorithm, \underline{d}iscontinuity-\underline{b}ounded \underline{A*} (db-A*), and trajectory optimization in an iterative fashion.
Db-A* generalizes A* with ideas from sampling-based planning to obtain solution trajectories that may have discontinuities up to a user-specified bound.
Our second contribution is the, to our knowledge, first benchmark that compares the three major kinodynamic motion planning techniques on the same problem instances with the identical objective of computing time-optimal trajectories.
While we focus in our evaluation on the challenging case of time-optimality, our approach supports arbitrary cost functions.

\section{Problem Description}

We consider a robot with state $\vx = [\vx^t, \vx^r] \in \sX \subset \mathbb R^{d_w} \times \mathbb R^{d_x-d_w}$, where the first $d_w$ dimensions indicate the translation in the workspace ($d_w\in\{2,3\}$) of the robot and the remaining $d_x - d_w$ dimensions may contain orientation or derivatives.
The robot can be actuated by controlling actions $\vu \in \sU \subset \mathbb R^{d_u}$.
We consider dynamics that are \emph{translation invariant}, with
\begin{equation}
    \label{eq:dynamics}
    \dot \vx = \vf(\vx^r, \vu),
\end{equation}
where $\vf$ only depends on $\vx^r$ and not on $\vx^t$.
In order to employ gradient-based optimization, we assume that we can compute the Jacobian of $\vf$ with respect to $\vx^r$ and $\vu$.

The robot is operating in a workspace $\sW \subseteq \mathbb R^{d_w}$ that indicates the free space for safe navigation.
The free state space then becomes $\sX_{\mathrm{free}} = \{\vx = [\vx^t, \vx^r] \in \sX | \vx^t \in \sW\}$.

Almost all generic kinodynamic motion planners assume a discrete-time formulation with zero-order hold, i.e., the applied action remains constant during a timestep.
We can then frame the dynamics \cref{eq:dynamics} as
\begin{equation}
    \label{eq:dynamics_discrete}
    \vx_{k+1} \approx \step(\vx_k, \vu_k) \equiv \vx_k + \vf(\vx^r_k, \vu_k)\Delta t,
\end{equation}
using a small timestep $\Delta t$ so that the Euler approximation holds sufficiently well.

Let $\seqX = \langle \vx_0, \vx_1, \ldots, \vx_T \rangle$ be a sequence of states sampled at times $0, \Delta t, \dots, T \Delta t$ and $\seqU = \langle \vu_0, \vu_1, \ldots, \vu_{T-1} \rangle$ be a sequence of actions applied to the system for times $[0,\Delta t), [\Delta t, 2\Delta t), \ldots, [(T-1)\Delta t, T\Delta t)$.
Then our goal of moving the robot from its start state to a goal state can be framed as the following optimization problem:
\begin{align}
    &\min_{\seqU,\seqX,T} J(\seqU,\seqX, T) \label{eq:motion-planning}
    \\
    &\text{\noindent s.t.}\begin{cases}
    \vx_{k+1} = \step(\vx_k, \vu_k) & \forall k \in \{0,\ldots,T-1\}\\
    \vu_k \in \sU & \forall k \in \{0,\ldots,T-1\}\\
    \vx_k \in \sX_{\mathrm{free}}  & \forall k \in \{0,\ldots,T\} \\
    \vx_0 = \vx_s;\,\,\vx_T = \vx_f, &
    \end{cases} \nonumber
\end{align}
where $\vx_s \in \sX$ is the start state and $\vx_f \in \sX$ is the goal state.
The objective function $J$ is application specific; we will focus on time-optimal trajectories, i.e., $J(\seqU,\seqX, T) = T\Delta t$.

\begin{example}
Consider a unicycle robot with state $\vx = [x, y, \theta] \in \sX = SE(2)\subset \mathbb R^{2} \times \mathbb R^{1}$, i.e., $x, y$ are the position and $\theta$ is the orientation.
The actions are $\vu = [v, \omega ]\in \sU \subset \mathbb R^{2}$, i.e., the speed and angular velocity can be controlled directly.
The dynamics are translation invariant: $\dot\vx = [v \cos \theta, v \sin \theta, \omega]$.
The choice of $\sU$ can make this low-dimensional problem challenging to solve.
For example, \cref{fig:overview} shows a plane-like case (positive minimum speed, i.e., $0.25 \leq v \leq \SI{0.5}{m/s}$) with a malfunctioning rudder (asymmetric angular speed, i.e., $-0.25 \leq \omega \leq \SI{0.5}{rad/s}$).
\end{example}

\section{Related Work}

There are several conceptually different algorithmic approaches to solving kinodynamic motion planning problems.

\textbf{Search-based} approaches rely on existing methods for discrete path planning, such as A* and variants.
The common approach is to generate short trajectories (\emph{motion primitives}) using a state lattice (i.e., pre-specified discrete state components)~\cite{pivtoraikoKinodynamicMotionPlanning2011}.
Each primitive starts and ends at a grid cell and swept cells can be precomputed for efficient collision checking.
Once motion primitives are computed, existing algorithms such as A* or the anytime variant ARA* can be employed without modification.

These methods can solve \cref{eq:motion-planning} if $\vx_s$ and $\vx_f$ fall within the chosen lattice and retain very strong theoretical guarantees on both optimality and completeness with respect to the chosen primitives.
The major challenge is to select and compute good motion primitives, especially for high-dimensional systems~\cite{PivtoraikoThesis,dispersionMinimizingPrimitives}.

\textbf{Sampling-based} approaches build a tree $\mathcal T$ rooted at the start state $\vx_s$.
During tree expansion, i) a random state $\vx_{\mathrm{rand}}$ in the state space is sampled, ii) an existing state $\vx_{\mathrm{expand}} \in \mathcal T$ is selected, and iii) a new state $\vx_{\mathrm{new}}$ is added with a motion that starts at $\vx_{\mathrm{expand}}$ and moves towards $\vx_{\mathrm{rand}}$.
The classic version of this approach, \emph{kinodynamic RRT}~\cite{kinodynamicRRT}, is probabilistically complete when using the correct variant~\cite{kunzKinodynamicRRTsFixed2015}.
Asymptotic optimality can be achieved when planning in state-cost space (\emph{AO-RRT})~\cite{AO-RRT,AO-RRT-Analysis, ST-RRT-Star} or by computing a sparse tree (\emph{SST*})~\cite{SSTstar}.
These methods rely on a distance function $d: \sX \times \sX \to \mathbb R$ and often use fast nearest neighbor data structures such as k-d trees for efficiency.
The mentioned algorithms work without solving a two-point boundary value problem, which is computationally expensive.

Sampling-based approaches are designed to explore the state space as fast as possible and typically do not use a heuristic function, unlike search-based methods.
The exploration/exploitation tradeoff is typically controlled by using goal-biasing instead.
These approaches cannot solve \cref{eq:motion-planning} directly, because the probability to reach $\vx_f$ by sampling is zero. Instead, the goal constraint is typically reformulated to $\vx_T \in \sX_f$ using a goal region $\sX_f$ rather than a goal state $\vx_f$.

\textbf{Optimization-based} approaches locally optimize an initial trajectory using the gradients of $J$, unlike the previous gradient-free methods. Dynamics, collision avoidance, goal constraints, and control limits are modeled with piece-wise differentiable functions. 
In \emph{CHOMP}, Hamiltonian Monte Carlo is used to perturb local solutions~\cite{CHOMP}, while \emph{TrajOpt}~\cite{TrajOpt} and \emph{GuSTO}~\cite{GuSTO,malyutaConvexOptimizationTrajectory2021} rely on sequential convex programming (SCP). Trajectories can also be computed with optimal control solvers that rely on Differential Dynamic Programming~
\cite{Crocoddyl} or extend the linear quadratic regulator to nonlinear systems ~\cite{li2004iterative}.

Both \emph{STOMP}~\cite{STOMP} and \emph{KOMO}~\cite{KOMO} use only the (geometric) state sequence $\seqX^t$ as decision variables and support kinodynamic systems via constraints.
All optimization-based approaches require an initial guess as a starting trajectory, but this guess does not need to be kinodynamically feasible.

These approaches can solve \cref{eq:motion-planning} directly for a differentiable $J$ and given number of timesteps $T$.
The observed solution quality is significantly higher (e.g., in terms of smoothness) compared to sampling-based or search-based approaches.
Moreover, optimization-based approaches do not suffer from the curse of dimensionality directly, although higher dimensions might result in more local optima.

\textbf{Hybrid} approaches combine two or more ideas.
One can combine search and optimization~\cite{natarajanInterleavingGraphSearch2021}, search and sampling~\cite{sakcakSamplingbasedOptimalKinodynamic2019,DIRT,shomeAsymptoticallyOptimalKinodynamic2021b}, or combine sampling and optimization~\cite{RABITstar}.
For some dynamical systems, using insights from control theory for the motion planning can also be beneficial \cite{LQR-RRTstar, kino-RRTstar, R3T}, but requires domain knowledge.
Motion planning can also benefit from using machine learning for computational efficiency~\cite{RL-RRT,L-SBMP}.

Our approach relies on the fact that the dynamics are translation-invariant for many mobile robots.
The most similar related works are a method that reuses edges within a sampling-based planning framework~\cite{shomeAsymptoticallyOptimalKinodynamic2021b} and a search-based approach that has an enhanced duplicate detection~\cite{duEscapingLocalMinima2019}.
Unlike those works we also include trajectory optimization and reuse precomputed and online computed motion primitives for fast convergence in practice.
Moreover, our discontinuity bound $\delta$ is not a fixed user-specified value but converges to 0, which allows us to provide stronger theoretical guarantees.

\section{Approach}

Our general approach is shown in \cref{alg:overview}.
We assume that we have access to a set of \emph{motion primitives}, which are valid trajectories according to our dynamics.
\begin{definition}
    \label{definition:motionPrimitive}
    A \emph{motion primitive} is a tuple $\langle \seqX, \seqU, T, c \rangle$ with
    \begin{equation}
    \begin{aligned}
        \seqX &= \langle \vx_0, \ldots, \vx_T\rangle & \vx_0^t = \vzero,\,\, \vx_k \in \sX\\
        \seqU &= \langle \vu_0, \ldots, \vu_{T-1}\rangle & \vu_k \in \sU\\
        \vx_{k+1} &= \step(\vx_k, \vu_k) & \forall k \in \{0,\ldots,T-1\}\\
        c &= J(\seqU, \seqX, T).
    \end{aligned}
    \end{equation}
\end{definition}
Typically, at least some of these are computed offline (e.g., using an optimization-based kinodynamic planner), while others can also be generated online.
There is no requirement on the optimality of each motion primitive, although \emph{optimal motion primitives}, where $J$ is minimized for the respective start and goal states, are beneficial in our setting.

Our kinodynamic planning approach iteratively improves the solution.
In every iteration, the following steps are executed: i) the set of used motion primitives grows and the bound that limits the maximum magnitude of discontinuous jumps that a solution may have is computed (\crefrange{alg:overview:sM}{alg:overview:delta});
ii) our discrete planner, db-A*, computes an initial solution that may include a bounded violation of some constraints (\cref{alg:overview:dbAstar});
iii) the result of db-A* warm-starts an optimization-based formulation (\cref{alg:overview:opt});
and iv) additional motion primitives are extracted from the optimization (\cref{alg:overview:extract}).

\begin{algorithm}[t]
    \caption{kMP-db-A*: Kinodynamic Motion Planning with db-A*}
    \label{alg:overview}
    \DontPrintSemicolon

    \SetKwFunction{AddPrimitives}{AddPrimitives}
    \SetKwFunction{ExtractPrimitives}{ExtractPrimitives}
    \SetKwFunction{ComputeDelta}{ComputeDelta}
    \SetKwFunction{DiscontinuityBoundedAstar}{db-A*}
    \SetKwFunction{Optimization}{Optimization}
    \SetKwFunction{Report}{Report}
    $\sM \leftarrow \emptyset$ \Comment*{Set of motion primitives}
    $c_{\mathrm{max}} \leftarrow \infty$ \Comment*{Solution cost bound}
    \For{$n=1,2,\ldots$}{
        $\sM \leftarrow \sM \cup \AddPrimitives()$\label{alg:overview:sM}\;
        $\delta \leftarrow \ComputeDelta(\sM)$\label{alg:overview:delta}\;
        $\seqX_d, \seqU_d, T_d \leftarrow$ \DiscontinuityBoundedAstar{$\vx_s, \vx_f, \sX_{\mathrm{free}}, \sM, \delta, c_{\mathrm{max}}$}\label{alg:overview:dbAstar}\;
        \If{$\seqX_d, \seqU_d$ successfully computed}{
            $\seqX, \seqU, T \leftarrow$ \Optimization{$\seqX_d, \seqU_d, T_d, c_{\mathrm{max}}$}\label{alg:overview:opt}\;
            \If{$\seqX, \seqU$ successfully computed}{
                \Report{$\seqX, \seqU, T$} \Comment*{New solution found}
                $c_{\mathrm{max}} \leftarrow \min(c_{\mathrm{max}}, J(\seqX, \seqU, T))$ \Comment*{cost bound}
            }
            $\sM \leftarrow \sM \cup \ExtractPrimitives(\seqX, \seqU)\label{alg:overview:extract}$\;
        }
    }
\end{algorithm}

\subsection{db-A*: Discontinuity-bounded A*}

In the following, we rely on a user-specified \emph{metric} $d: \sX \times \sX \to \mathbb R$, which measures the distance between two states.
This is analogous to sampling-based planners, and we assume that $\langle \sX, d\rangle$ is a metric space in order to use efficient nearest neighbor data structures, such as k-d trees.

\begin{definition}
\label{definition:discontinuityBounded}
Sequences $\seqX = \langle \vx_0, \ldots, \vx_T\rangle$, $\seqU = \langle \vu_0, \ldots, \vu_{T-1}\rangle$ are \emph{$\delta$-discontinuity-bounded} solutions to \cref{eq:motion-planning} iff the following conditions hold:
\begin{subequations}
\begin{align}
    d(\vx_{k+1}, \step(\vx_k, \vu_k)) \leq \delta \quad k \in \{0,\ldots, T-1\} \label{eq:disBound:dynamics}\\
    \vu_k \in \sU \quad \forall k \in \{0,\ldots, T-1\} \label{eq:disBound:U}\\
    \vx_k \in \sX_{\mathrm{free}} \quad \forall k \in \{0,\ldots, T\} \label{eq:disBound:Xfree}\\
    d(\vx_0, \vx_s) \leq \delta \label{eq:disBound:xs}\\
    d(\vx_T, \vx_f) \leq \delta. \label{eq:disBound:xf}
\end{align}
\end{subequations}

\end{definition}

Intuitively, \cref{definition:discontinuityBounded} enforces that the sequences connect the start and goal states with a bounded error $\delta$ in the dynamics, which corresponds to ``stitching'' primitives together. By the definition of a metric space, $\seqX$ and $\seqU$ fulfill all constraints of \cref{eq:motion-planning} if $\delta = 0$.

Our approach to compute such sequences is \emph{discontinuity-bounded A*} (db-A*), see \cref{alg:dbAstar}.
Db-A* is, like A*, an informed search that relies on a \emph{heuristic} $h: \sX \to \mathbb R$ to explore an implicitly defined directed graph efficiently.
Nodes in the graph represent states and an edge between nodes indicates a $\delta$-bounded motion that connects the states.

The algorithm keeps track of nodes to explore using a \emph{priority queue}, which is sorted by the lowest $f(\vx)=g(\vx)+h(\vx)$ value, where $g(\vx)$ is the cost-to-come.
The overall structure is the same as in A*: The OPEN priority queue $\mathcal O$ is initialized with the start state (\cref{alg:dbAstar:Oinit}), the current node $n$ is the removed first element of $\mathcal O$ (\cref{alg:dbAstar:Opop}), and that node is expanded in order to compute valid (collision-free) neighbors (\crefrange{alg:dbAstar:expand1}{alg:dbAstar:expand2}).
Newly found nodes are added directly to $\mathcal O$ (\cref{alg:dbAstar:Oadd}), while previously found nodes are updated if the solution cost is reduced (\crefrange{alg:dbAstar:Oupdate1}{alg:dbAstar:Oupdate2}).

Unlike A*, we consider two states to be identical for nonzero $\delta$ values.
The major changes of db-A* from A* are highlighted in \cref{alg:dbAstar}.
We use the notation $\vx \oplus m$ to indicate that a motion $m$ is applied to state $\vx$; that is, we shift $m$ by the translational part of $\vx$.
For efficient search, we adopt two k-d trees (rather than a hashmap in A*).
Namely, we use $\mathcal T_m$ (\cref{alg:dbAstar:Tm}) to index the start states of all provided motion primitives, which can be done once at the beginning.
In order to reuse the same distance metric $d$, we use the translation-invariance property and set $\vx^t$ of a given state $\vx$ to $\vzero$.
This data structure allows us to efficiently find suitable motions extending from a given state (\cref{alg:dbAstar:Mprime}).
The second k-d tree $\mathcal T_n$ (\cref{alg:dbAstar:Tn}) contains the states of all explored nodes and grows dynamically (\cref{alg:dbAstar:TnGrow}).
It is used to find nearby previously explored states (\cref{alg:dbAstar:Nprime}) in order to limit the graph size and enable rewiring.
The discontinuity with a magnitude of up to $\delta$ may occur in two cases.
First, when we select suitable motion primitives for expansion (\cref{alg:dbAstar:Mprime}) and second, when we prune a potential new node in favor of already existing states (\cref{alg:dbAstar:Nprime}).
The tradeoff between the two can be selected by a user-specified parameter $\alpha \in (0, 1)$.
For most search-based algorithms, collision checking is achieved using a binary occupancy grid, which makes the choice of the grid size a critical decision.
Instead, we rely on broadphase collision checking.
The required data structures can be efficiently precomputed for the environment and each motion primitive.
For the collision check in \cref{alg:dbAstar:collision}, we only need to shift the data structure for the selected motion primitive, before executing the broadphase collision check.

\begin{algorithm}[t]
    \caption{db-A*}
    \label{alg:dbAstar}
    \DontPrintSemicolon

    \SetKwFunction{NearestNeighborInit}{NearestNeighborInit}
    \SetKwFunction{NearestNeighborQuery}{NearestNeighborQuery}
    \SetKwFunction{NearestNeighborAdd}{NearestNeighborAdd}
    \SetKwFunction{PriorityQueuePop}{PriorityQueuePop}
    \SetKwFunction{PriorityQueueInsert}{PriorityQueueInsert}
    \SetKwFunction{PriorityQueueUpdate}{PriorityQueueUpdate}

    \KwData{$\vx_s, \vx_f, \sX_{\mathrm{free}}, \sM, \delta, c_{\mathrm{max}}$}
    \KwResult{$\seqX_d, \seqU_d$ or Infeasible}
    
    \tikzmk{A}$\mathcal T_m \leftarrow \NearestNeighborInit(\mathcal M)$ \Comment*{Use start states of motions (excl. position)} \label{alg:dbAstar:Tm}
    $\mathcal T_n \leftarrow \NearestNeighborInit(\{\vx_s\})$ \Comment*{capture explored vertices (incl. position)}\label{alg:dbAstar:Tn}\marklineFour{-10pt}
    $\mathcal O \leftarrow \{Node(\vx: \vx_s, g: 0, h: h(\vx_s), p: None, a: None) \}$ \label{alg:dbAstar:Oinit} \Comment*{Initialize open priority queue}
    \While{$|\mathcal O| > 0$}{
        $n \leftarrow \PriorityQueuePop(\mathcal O)$ \Comment*{Lowest f-value} \label{alg:dbAstar:Opop}
        \tikzmk{A} \If{$d(n.\vx, \vx_f) \leq \delta$\marklineOne{-17pt} \label{alg:dbAstar:sol_cond}}{
            \Return $\seqX_d, \seqU_d, T_d$ \Comment*{Lowest f-value} \label{alg:dbAstar:sol}
        }
        \tikzmk{A}\Comment{Find applicable motion primitives with discontinuity up to $\alpha \delta$}
        $\mathcal M' \leftarrow \NearestNeighborQuery(\mathcal T_m, n.\vx^r, \alpha\delta)$\marklineThree{-17pt} \label{alg:dbAstar:Mprime}\;
        \ForEach{$m\in \mathcal M'$ \label{alg:dbAstar:expand1}}{
        \If{$n.\vx\oplus m \notin \sX_{\mathrm{free}}$ \label{alg:dbAstar:collision}}{
            \Continue \label{alg:dbAstar:expand2} \Comment*{entire motion is not collision-free}
        }
        $g_t \leftarrow n.g + cost(m)$ \Comment*{tentative g score for this action}
        \tikzmk{A}\Comment{find already explored nodes within $(1-\alpha)\delta$}
        $\mathcal N' \leftarrow \NearestNeighborQuery(\mathcal T_n, n.\vx\oplus m, (1-\alpha)\delta)$\marklineThree{-24pt} \label{alg:dbAstar:Nprime}\;
        \eIf{$\mathcal N' = \emptyset$}{
            $\PriorityQueueInsert(\mathcal O, Node(\vx: n.\vx \oplus m, g: g_t, h: h(n.\vx\oplus m), p: n, a: m))$ \label{alg:dbAstar:Oadd}\;
            \tikzmk{A}$\NearestNeighborAdd(\mathcal T_n, n.\vx \oplus m)$\marklineOne{-32pt} \label{alg:dbAstar:TnGrow}\;
        }{
            \ForEach{$n'\in \mathcal N'$}{
                \If{$g_t < n'.g$ \Comment*{This motion is better than a known motion} \label{alg:dbAstar:Oupdate1}}{
                    $n'.g = g_t$ \Comment*{Update cost}
                    $n'.p = n$ \Comment*{Update parent}
                    $n'.a = m$ \Comment*{Update action}
                    $\PriorityQueueUpdate(\mathcal O, n')$ \label{alg:dbAstar:Oupdate2}
                }
            }
        }
        }
    }
    \Return Infeasible
\end{algorithm}

We now discuss the theoretical properties of db-A*.

\begin{theorem}
    Sequences $\seqX$ and $\seqU$ returned by db-A* (\cref{alg:dbAstar}) are a $\delta$-discontinuity-bounded solution to the given motion planning problem.
\end{theorem}
\begin{proof}
    \Cref{alg:dbAstar} only returns a sequence in \cref{alg:dbAstar:sol}. Due to the condition in \cref{alg:dbAstar:sol_cond}, \cref{eq:disBound:xf} holds.

    By \cref{definition:motionPrimitive}, we have $d(\vx_{k+1}, \step(\vx_k, \vu_k)) = 0 \leq \delta$, $\vu_k \in \sU$, and $\vx_k \in \sX$ for each motion primitive $m\in \mathcal M$.
    Thus, \cref{eq:disBound:U} holds.
    During the search, we expand motions whose start states are at most $\alpha\delta$ away from the current state $n.\vx$ (\cref{alg:dbAstar:Mprime}).
    There are two cases.
    First, the motion corresponds to an edge leaving from the current state $n.\vx$ (\cref{alg:dbAstar:Oadd}), in which case we have $d(m^0, n.\vx) \leq \alpha\delta$, where $m^0$ is the (translated) first state of motion $m$.
    Second, the motion becomes an edge leaving from some neighbor state $n'.\vx$ that is at most $(1-\alpha)\delta$ away from $n.\vx$ (\cref{alg:dbAstar:Nprime}), in which case we have $d(m^0, n'.\vx) \leq \alpha\delta + (1-\alpha)\delta = \delta$, using the triangle inequality of our metric space.
    Thus, \cref{eq:disBound:dynamics} holds for all edges.

    We already know that $\vx_k \in \sX$. Motions are only used as edges, if the entire motion is in $\sX_{\mathrm{free}}$ (\cref{alg:dbAstar:collision}), thus \cref{eq:disBound:Xfree} holds.
    Finally, \cref{eq:disBound:xs} holds because $\mathcal O$ is initialized with $\vx_s$ (\cref{alg:dbAstar:Oinit}) and \cref{eq:disBound:dynamics} holds.
\end{proof}

\begin{remark}
    \label{th:incomplete}
    Db-A* is incomplete and suboptimal if $\delta > 0$.
\end{remark}
\begin{proof}
    Consider an example where a robot has to move through a narrow door to navigate to an adjacent room.
    Even if a $\delta$-discontinuity-bounded solution for the problem exists, db-A* may not find it, because motions are added in a random order and only if no other node is within $(1-\alpha)\delta$ (\cref{alg:dbAstar:Nprime}).
    Since db-A* is incomplete, it cannot guarantee that no better $\delta$-discontinuity-bounded solution exists, once it finds one.
\end{proof}

We note that \cref{th:incomplete} uses a very strong definition of completeness in continuous state space.
Other possible definitions include \emph{$\delta$-robust completeness}~\cite{SSTstar}; we leave the analysis regarding that property to future work.
For the purpose of this paper, it is important to recognize that the strong properties of A* hold in the limit, i.e., as $\delta \to 0$. 

\subsection{Kinodynamic Optimization}
\label{sec:approach:opt}
For the \texttt{Optimization} subroutine, we rely on \emph{$k$-Order Motion Optimization} (KOMO)~\cite{KOMO}, which solves the following optimization problem:
\begin{align}
    &\min_{\seqX} \sum_{l=1}^T \hat{J}(\vx_{l-k:l}) \label{eq:komo}
    \\
    &\text{\noindent s.t.}\begin{cases}
    \vg_l(\vx_{l-k:l}) \leq \vzero & \forall l \in \{1,\ldots,T\}\\
    \vh_l(\vx_{l-k:l}) = \vzero & \forall l \in \{1,\ldots,T\}
    \end{cases}. \nonumber
\end{align}
Here, $\vx_{l-k:l}$ denotes the sequence $\vx_{l-k}, \vx_{l-k+1}, \ldots, \vx_{l}$ and the inequality constraints $\vg_l$ and equality constraints $\vh_l$ only depend on the current and up to $k$ prior states.
This $k$-order Markov assumption allows us to solve the nonlinear optimization problem efficiently e.g., using the augmented Lagrangian method, because $k$ is typically small (1 to 3).

When using the Euler approximation in \cref{eq:dynamics_discrete}, we can transform \cref{eq:motion-planning} for a given $T$ into \cref{eq:komo} by encoding the dynamics, start, and goal constraints using $\vh_l$ and the action and state constraints into $\vg_l$.
Since $\seqU$ is not a decision variable in this formulation, the dynamics constraint has to be encoded by using state constraints or by augmenting the state space.
We note that if $T$ and $\Delta t$ are fixed and $J=T \Delta t$, we can use any $\hat{J}$ to optimize in the nullspace of $J$.
This allows us to include arbitrary regularization terms (in our case, smoothness)  to guide the optimization and improve the convergence and success rate of the optimizer.

Some optimization methods may refine the given $T_d$ in \cref{alg:overview:opt} of \cref{alg:overview} either by adding $\Delta t$ as an optimization variable~\cite{malyutaConvexOptimizationTrajectory2021, ponton2021efficient}
(which introduces additional nonlinearities), or by applying a linear search over multiple potential values of $T$ that are around $T_d$, e.g., $T\in \langle 0.8 T_d, T_d, 1.2 T_d\rangle$.
We use the latter approach for \cref{alg:overview}.

When no estimate of $T$ is available, we can use a linear search over $T$.
For some dynamics, e.g., differentially-flat systems, it is also possible to use a modified binary search, where the first exponential search identifies an upper bound and the following binary search finds the optimal $T$.
We use the latter approach for our baseline.

\subsection{Motion Primitive Generation}

Instead of sampling control sequences at random, we solve two-point boundary value problems with random start and goal configurations in free space with nonlinear optimization, which results in a superior primitive distribution.
Specifically, we generate motion primitives offline using the following steps.
First, random sampling of a start and goal configuration in free space;
second, solving \cref{eq:komo} using linear search over $T$; and third, splitting the resulting motion into multiple pieces of a desired length.
We sort the primitives using an iterative greedy method that approximately minimizes the dispersion.
Let $\sM$ be the set of all motions, $\sM_s$ be the set of sorted motions, and $\sM_r = \sM \setminus \sM_s$ be the set of remaining motions.
We initialize $\sM_s = \{\argmax_{m\in\sM} d(m^0, m^f) \}$, where $m^0$ refers to the initial state of the motion and $m^f$ to the final state of the motion.
Then, we add an element to $\sM_s$ in each iteration selected by
\begin{equation}
    \argmax_{m_r\in \sM_r} \left(\min_{m_s\in\sM_s} d(m_r^0, m_s^0) + \min_{m_s\in\sM_s} d(m_r^f, m_s^f)\right).
\end{equation}
Thus, we pick the motion in each iteration that maximizes the minimum distance to other, already picked motions.

For \texttt{AddPrimitives} we add motions from the precomputed sequence $\sM_s$.
Additional motions can be generated online using the same procedure.

Instead of letting users manually specify $\delta$, we use the automatic procedure \texttt{ComputeDelta}, which estimates $\delta$ given a desired branching factor $b_d$.
First, we initialize a k-d tree $\mathcal T_m$ of all motions, as in \cref{alg:dbAstar:Tm} of \cref{alg:dbAstar}.
Second, we randomly sample a state $\vx_{\mathrm{rand}}$.
Third, we use $\mathcal T_m$ to find the $b_d$-closest motions that could be applied from $\vx_{\mathrm{rand}}$.
Fourth, we record the distance $\delta_r = \max d(m, \vx_{\mathrm{rand}})$, where $m$ is one of the $b_d$-closest motions.
The estimated value of $\delta$ is the average over multiple $\delta_r$ values.
This procedure reduces $\delta$ as the number of motion primitives increases in expectation and is easy to tune at the same time.

Motion primitives can also be extracted online in \cref{alg:overview}.
The \texttt{ExtractPrimitives} procedure uses the output of the optimization regardless of the constraint satisfaction and works as follows.
First, intervals of valid sub-trajectories are computed by checking if all the constraints are fulfilled.
Longer intervals can be split up as in the offline computation.
The resulting primitives can be particularly useful for the planning problem at hand, because they are computed using the full knowledge of the environment.

\subsection{Properties}

We conjecture that the approach in \cref{alg:overview} will eventually compute the optimal solution, because as the number of iterations $n$ increases, we add more primitives $\mathcal M$, which, by definition of \texttt{ComputeDelta}, reduces $\delta$.
Thus, as $n\to\infty$, we have $\delta\to 0$.
For $\delta = 0$, db-A* as described in \cref{alg:dbAstar} becomes regular A*, which is known to be complete and optimal.
The major flaw of this argument is that, in the limit, we also have an infinite number of motion primitives and thus an infinite branching factor.

Formally, we can follow \cite[Th. 3]{AO-RRT} to establish almost surely asymptotic optimality (which also implies probabilistic completeness) under the assumption that we have a non-zero probability of our \texttt{Optimization} method to find a solution if one exists.
This assumption is justified by the fact that the nonlinear trajectory optimization has a region of attraction $\Delta > 0$ and for small $\delta > 0$ our initial guess will fall in this region of attraction, allowing the optimization method to eventually compute a solution if one exists.

\begin{theorem}
    \label{theorem:ao}
    The kMP-db-A* motion planner in \cref{alg:overview} is asymptotically optimal, i.e.
    \begin{equation}
        \lim_{n\to\infty} P(\{ c_n - c^* > \epsilon \}) = 0, \; \forall \epsilon > 0,
    \end{equation}
    where $c_n$ is the best cost in iteration $n$ and $c^*$ is the optimal cost.
\end{theorem}
\begin{proof}
    We closely follow \cite[Th. 3]{AO-RRT}.
    Let $S_1,\ldots,S_n$ be random variables denoting the suboptimality $c_n - c*$.
    In every iteration of \cref{alg:overview} we either reduce the cost if we find a new solution, or we remain with the same cost, i.e., $c_{n+1} \leq c_{n}$.
    Each iteration, we add more motion primitives and thus reduce $\delta$.
    Using the assumption of the non-zero probability for our \texttt{Optimization} method to find a solution, we have
    $E[S_n|s_{n-1}] \leq (1-\omega)s_{n-1}$, i.e., in expectation the solution improves by a constant amount $\omega > 0$ every iteration.
    Then, we have
    \begin{align}
        E[S_n] = \int E[S_n | s_{n-1}] P(s_{n-1}) ds_{n-1}\\
        \leq (1-\omega) \int s_{n-1} P(s_{n-1}) ds_{n-1}\nonumber\\
        = (1-\omega)E[S_{n-1}] = (1-\omega)^{n-1} E[S_1].\nonumber
    \end{align}
    With the Markov inequality we have $P(S_n > \epsilon) \leq E[S_n]/\epsilon = (1-\omega)^{n-1} E[S_1] / \epsilon$, which approaches 0 as $n\to\infty$.
\end{proof}

\section{Experimental Results}

We compare different motion planners, including ours, on the same problem scenarios.
For fair comparison, we share code and data structures as much as possible, use the respective state-of-the-art open-source implementations, and focus on settings where the dynamics and not the collision-checking create challenges.

\subsection{Dynamical Systems}

\textbf{Unicycle ($1^{\text{st}}$ order)} has a 3-dimensional state space $[x,y,\theta] \in SE(2)$ and a 2-dimensional $[v, \omega] \in \sU \subset \mathbb R^2$ control space with dynamics defined in \cite[Eq. (13.18)]{lavallePlanningAlgorithms2006}.
The simplest version (v0) uses bounds $v \in [-0.5, 0.5]~\si{m/s}$ and $\omega \in [-0.5, 0.5]~\si{rad/s}$.
More interesting variants are a plane-like version (v1) using a positive minimum speed of \SI{0.25}{m/s}, and a plane-like version with a rudder damage (v2) ($\omega \in [-0.25, 0.5]~\si{rad/s}$).

\textbf{Unicycle ($2^{\text{nd}}$ order)} has a 5-dimensional state space $[x,y,\theta,v,\omega] \in \sX \subset \mathbb R^5$, a 2-dimensional $[\dot v, \dot\omega] \in \sU \subset \mathbb R^2$ control space, and dynamics defined in \cite[Eq. (13.46)]{lavallePlanningAlgorithms2006}.
Our version (v0) uses $v \in [-0.5, 0.5]~\si{m/s}$, $\omega \in [-0.5, 0.5]~\si{rad/s}$, $\dot v \in [-0.25, 0.25]~\si{m/s^2}$, and $\dot\omega \in [-0.25, 0.25]~\si{rad/s^2}$.

\textbf{Car with trailer} has a 4-dimensional state space $[x,y,\theta_0, \theta_1] \in \sX \subset \mathbb R^4$, a 2-dimensional $[v, \phi] \in \sU \subset \mathbb R^2$ control space, and dynamics and visualization given in \cite[Eq. (13.19), Fig. 13.6]{lavallePlanningAlgorithms2006}.
We add an additional constraint $|\angle(\theta_0, \theta_1)| < \pi / 4$ that avoids that the angle between the car and the trailer exceeds a threshold.
Our version (v0) uses $v \in [-0.1, 0.5]~\si{m/s}$, $\phi \in [-\pi/3, \pi/3]$, $L=\SI{0.25}{m}$, and $d_1=\SI{0.5}{m}$, where $L$ and $d_1$ are defined in~\cite{lavallePlanningAlgorithms2006}.

\textbf{Quadrotor} has a 13-dimensional state space (pose and first order derivatives using a Quaternion representation), a 4-dimensional control space (force for each of the four motors), and dynamics defined in \cite[Eq. (1)]{neuralswarm}.
We use the parameters of the Crazyflie quadrotor with limits on the motor forces, velocity, and angular velocity.
Note that the low thrust-to-weight ratio of \num{1.4} is very challenging for kinodynamic motion planning and that problem settings with a harsh initial condition prevent the use of specialized methods~\cite{liuSearchBasedMotionPlanning2018,zhouRobustEfficientQuadrotor2019}.

We use $\Delta t=\SI{0.1}{s}$ for all dynamical systems except the quadrotor, which uses $\Delta t=\SI{0.01}{s}$ due to the fast rotational dynamics.

\subsection{Environments}

For most of the dynamical systems, we consider three environments (see \cref{fig:env}), which are inspired by the common use-cases in the related literature. For the v2 unicycle, we use the \emph{wall} environment as shown in \cref{fig:overview}.
For the quadrotor, we use an \emph{empty} environment without obstacles.
The scenario requires the quadrotor to recover from a harsh initial condition with an upside-down initial rotation and nonzero initial first derivatives.
All environments only use simple geometric box shapes for efficient collision checking.
The environments are bounded, where the bounds only limit the translational part of the state, i.e., parts of the robots are allowed to be outside. One such example is visible in the park solution of \cref{fig:env}.

\begin{figure}
    \centering
    \includegraphics[width=0.32\linewidth]{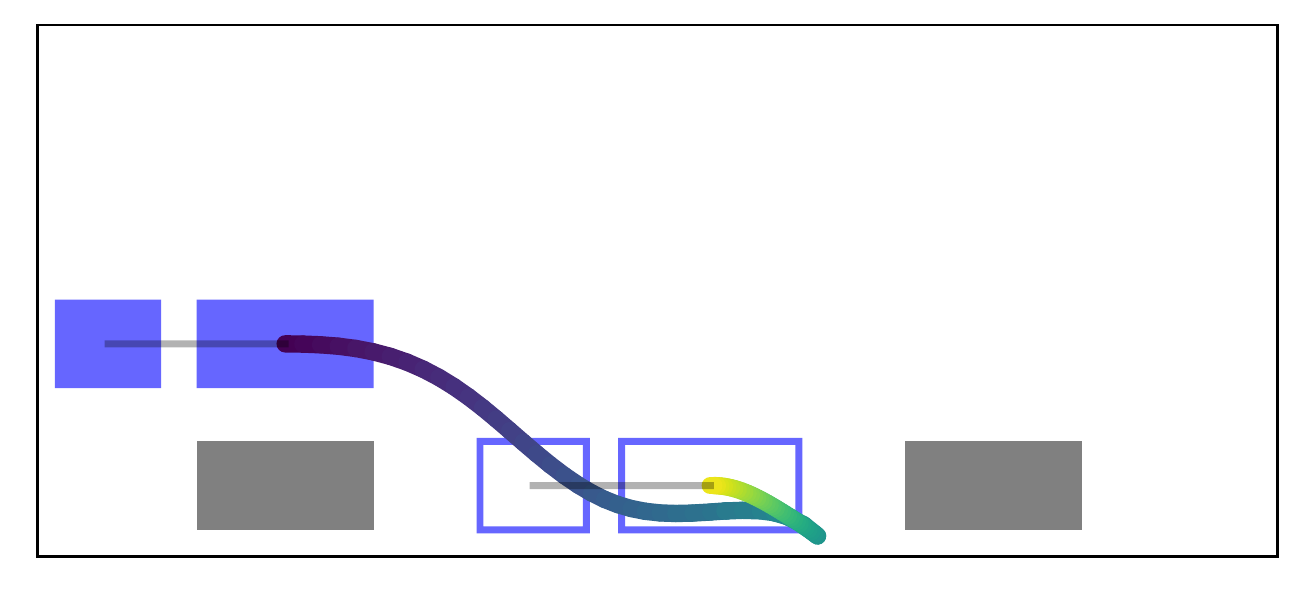}
    \includegraphics[width=0.32\linewidth]{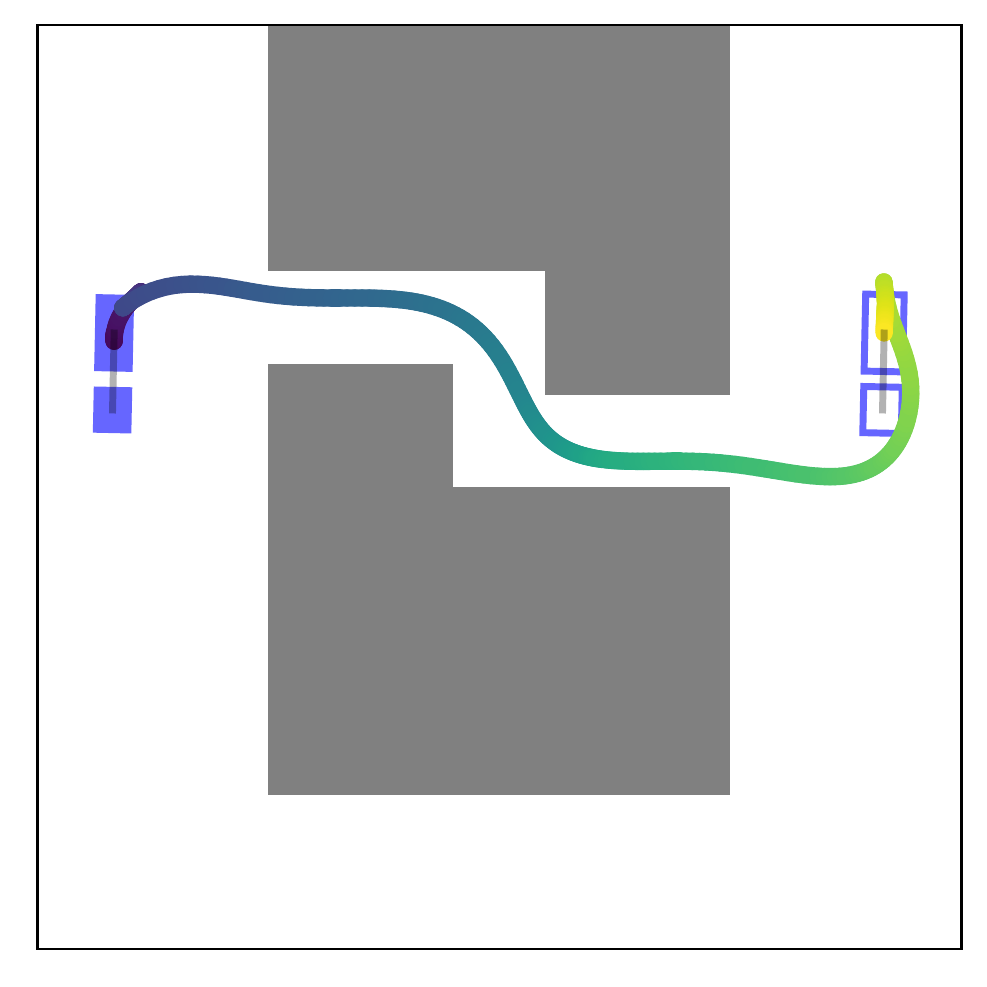}
    \includegraphics[width=0.32\linewidth]{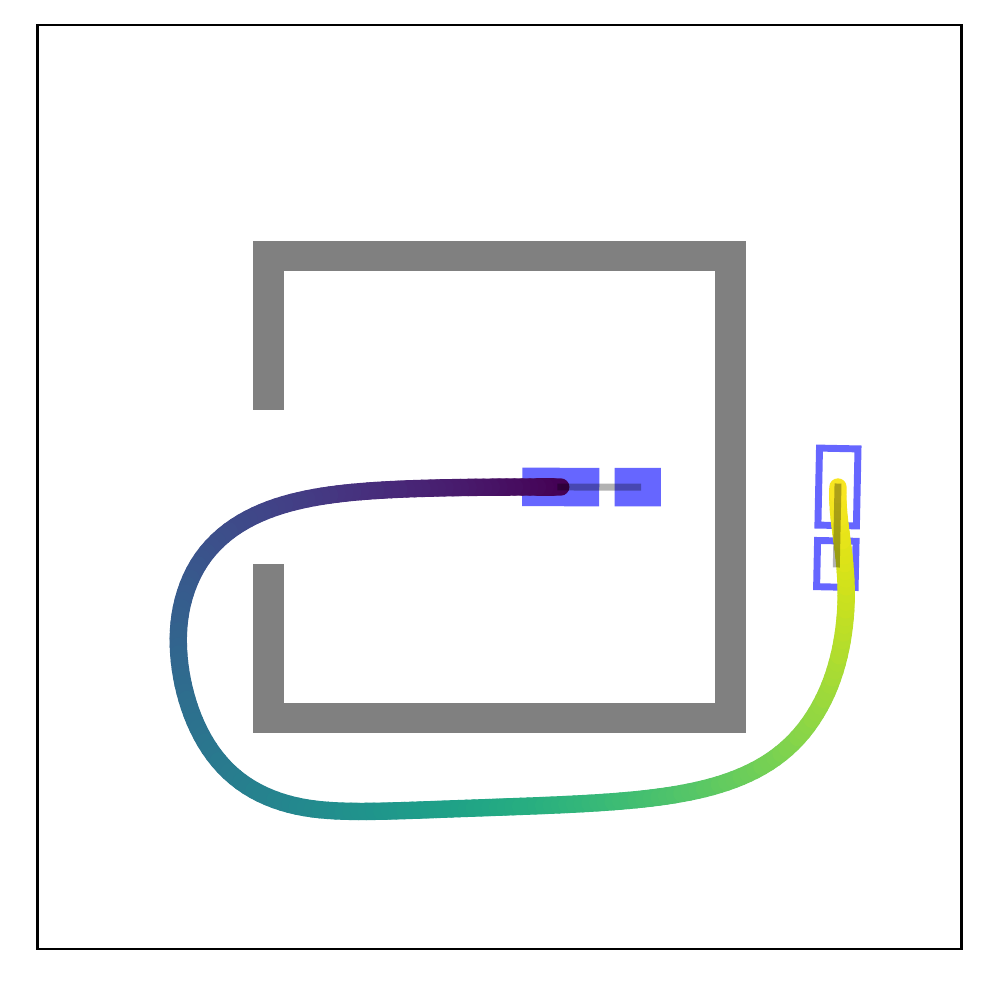}
    \caption{
    Example environments park, kink, and bugtrap (left to right) for the car with trailer dynamics. The start state is indicated in filled blue, the goal with a blue outline, and a near-optimal solution of the $(x,y)$-part of the state as computed by kMP-db-A* as a time-colored line.
    }
    \label{fig:env}
\end{figure}

\subsection{Algorithms}

\begin{table*}
    \caption{Benchmark results comparing success rate ($p$), time for the first found solution ($t^{\mathrm{st}}$), the cost of the first found solution ($J^{\mathrm{st}}$), and the cost of the solution after \SI{5}{min} ($J^{f}$). Time and cost are the median over 10 trials. Best results are \textbf{bold}.}
    \label{tab:results}

    \begin{center}
    \setlength{\tabcolsep}{0.3em} %
    \renewcommand{\arraystretch}{1.1}%
    \begin{tabular}{c || c|c || r|r|r|r || r|r|r|r || r|r|r|r || r|r|r|r}
        \# & System & Instance & \multicolumn{4}{c||}{SST*} & \multicolumn{4}{c||}{SBPL} & \multicolumn{4}{c||}{geom. RRT*+KOMO} & \multicolumn{4}{c}{kMP-db-A*}\\
        &&  & $p$ & $t^{\mathrm{st}} [s]$ & $J^{\mathrm{st}} [s]$ & $J^{f} [s]$ & $p$ & $t^{\mathrm{st}} [s]$ & $J^{\mathrm{st}} [s]$ & $J^{f} [s]$ & $p$ & $t^{\mathrm{st}} [s]$ & $J^{\mathrm{st}} [s]$ & $J^{f} [s]$ & $p$ & $t^{\mathrm{st}} [s]$ & $J^{\mathrm{st}} [s]$ & $J^{f} [s]$\\
        \hline
        \hline
        1 & \multirow{3}{*}{unicycle $1^{\mathrm{st}}$ order, v0} & park  & \bfseries 1.0 & 1.2 & 5.9 & 3.5 & \bfseries 1.0 & \bfseries 0.0 & 6.2 & 6.2 & \bfseries 1.0 & 3.1 & 5.3 & 3.2 & \bfseries 1.0 & 0.9 & \bfseries 3.2 & \bfseries 3.1\\
        2 &  & kink  & \bfseries 1.0 & 1.2 & 47.5 & 17.9 & \bfseries 1.0 & \bfseries 0.2 & 22.6 & 22.6 & \bfseries 1.0 & 9.9 & 24.8 & 21.7 & \bfseries 1.0 & 5.5 & \bfseries 15.4 & \bfseries 13.1\\
        3 &  & bugtrap  & \bfseries 1.0 & \bfseries 1.1 & 63.3 & 30.0 & \bfseries 1.0 & 1.4 & 36.8 & 36.6 & \bfseries 1.0 & 10.8 & 40.3 & 22.2 & \bfseries 1.0 & 21.6 & \bfseries 23.8 & \bfseries 22.1\\
        \hline
        4 & \multirow{1}{*}{unicycle $1^{\mathrm{st}}$ order, v1} & kink  & 0.8 & 38.5 & 43.2 & 34.3 &&&&  & 0.0 & \textemdash & \textemdash & \textemdash & \bfseries 1.0 & \bfseries 11.9 & \bfseries 23.9 & \bfseries 23.7\\
        \hline
        5 & \multirow{1}{*}{unicycle $1^{\mathrm{st}}$ order, v2} & wall  & 0.8 & 29.9 & 45.2 & 37.4 &&&&  & 0.0 & \textemdash & \textemdash & \textemdash & \bfseries 1.0 & \bfseries 7.3 & \bfseries 20.0 & \bfseries 18.0\\
        \hline
        6 & \multirow{3}{*}{unicycle $2^{\mathrm{nd}}$ order} & park  & \bfseries 1.0 & \bfseries 4.7 & 14.4 & 7.5 &&&&  & \bfseries 1.0 & 5.5 & 10.6 & \bfseries 6.1 & \bfseries 1.0 & 7.9 & \bfseries 6.8 & \bfseries 6.1\\
        7 &  & kink  & \bfseries 1.0 & \bfseries 2.6 & 71.0 & 59.7 &&&&  & 0.5 & 18.1 & 24.6 & 21.1 & \bfseries 1.0 & 10.7 & \bfseries 20.4 & \bfseries 19.8\\
        8 &  & bugtrap  & \bfseries 1.0 & \bfseries 5.2 & 66.8 & 51.1 &&&&  & \bfseries 1.0 & 23.2 & 39.9 & 27.2 & \bfseries 1.0 & 40.6 & \bfseries 33.9 & \bfseries 25.9\\
        \hline
        9 & \multirow{3}{*}{car with trailer} & park  & 0.6 & 24.6 & 13.6 & 13.6 &&&&  & \bfseries 1.0 & 15.6 & 10.8 & 6.2 & \bfseries 1.0 & \bfseries 10.0 & \bfseries 5.7 & \bfseries 5.4\\
        10 &  & kink  & \bfseries 1.0 & \bfseries 9.1 & 70.8 & 66.5 &&&&  & 0.1 & 217.5 & 174.4 & 130.8 & \bfseries 1.0 & 94.8 & \bfseries 34.1 & \bfseries 24.2\\
        11 &  & bugtrap  & \bfseries 1.0 & \bfseries 0.7 & 47.5 & 43.8 &&&&  & 0.0 & \textemdash & \textemdash & \textemdash & \bfseries 1.0 & 8.3 & \bfseries 21.5 & \bfseries 19.4\\
        \hline
        12 & \multirow{1}{*}{quadrotor} & empty  & 0.0 & \textemdash & \textemdash & \textemdash &&&&  & \bfseries 1.0 & 207.2 & 2.6 & 2.6 & \bfseries 1.0 & \bfseries 131.0 & \bfseries 1.6 & \bfseries 1.6\\
        \end{tabular}
    \end{center}
\end{table*}

For a \textbf{search-based} approach, we rely on SBPL\footnote{\url{https://github.com/sbpl/sbpl}} (Search-based Planning Library), a commonly used C++ library with integration in the Robot Operating System (ROS).
SBPL contains an example for unicycles, although the used dynamics do not match the ones from \cite[eq. 13.18]{lavallePlanningAlgorithms2006}.
Thus, we generate our own primitives using the formulation in \cref{sec:approach:opt}.
Moreover, we make minor adjustments to the heuristic to enable time-optimal anytime planning using the provided implementation of ARA* in SBPL.
Due to limits in SBPL\footnote{The official documentation states: ``[For custom scenarios], you will have to implement your own environment (a very involved topic that might be covered in the future).'' \url{http://sbpl.net/node/47}}, we limit our evaluation to the v0 first order unicycle.

For a \textbf{sampling-based} approach, we rely on OMPL~\cite{OMPL} (Open Motion Planning Library), a widely used C++ library with integration in ROS through MoveIt.
OMPL implements several kinodynamic planners, including SST*~\cite{SSTstar}, which we use.
As part of this work, we contribute minor changes to allow time as an optimization objective.
Since sampling-based kinodynamic approaches cannot reach a goal state, we use a goal region instead that we verify to be small enough such that an optimizer can find an exact solution.

For an \textbf{optimization-based} approach, we rely on RAI\footnote{\url{https://github.com/MarcToussaint/rai}} (Robotic AI), a C++ library that implements KOMO and nonlinear optimization algorithms.
For each of the dynamical systems, we implement the appropriate constraints and their derivative computation.
In case of the trailer and the quadrotor, we add parts of the actions as decision variables (angle $\phi$ and motor forces, respectively); otherwise the decision variables are the state sequences only.
As an initial guess, we use a geometric solution as found by RRT* of OMPL.
We then use the modified binary search method as outlined in \cref{sec:approach:opt}.
This combination of \emph{geometric RRT*+KOMO} is anytime like the other approaches we compare to. 

For \textbf{db-A*}, we implement \cref{alg:dbAstar} in C++ using the data structures provided in OMPL to represent states and for nearest neighbor computation.
\Cref{alg:overview} is implemented in Python that executes C++ binaries for subroutines when necessary.
As heuristic $h$, we use Euclidean distance divided by the upper bound of the speed.
For the \texttt{AddPrimitives} function, we precompute \num{10000} motion primitives for most dynamical systems (\num{30000} for the quadrotor) and only add a subset per iteration.
Generating the primitives took about \SI{8}{h} per dynamical system utilizing all CPU cores.

The benchmark infrastructure is written in Python and all tuning parameters can be found in the open-source repository\footnote{\url{https://github.com/IMRCLab/kinodynamic-motion-planning-benchmark}}.
Collision checking is done using FCL (Flexible Collision Library)~\cite{FCL} in all cases.
All approaches use the Euler integration \cref{eq:dynamics_discrete}, although KOMO uses an implicit formulation by design.

\subsection{Benchmark}

We execute our benchmark on a desktop computer with AMD Ryzen 9 3900X (\SI{3.8}{GHz}) and \SI{32}{GB} RAM.
Our results are summarized in \cref{tab:results}.

We summarize the main results as follows.
\textbf{SBPL} can compute results very quickly and consistently. The initial solution quality is very high, but due to the limited number of primitives, the solution does not improve much over time (rows 1 -- 3, \cref{tab:results}).
The approach is not as general as the other ones, and we were unable to use it for all of our dynamical systems.
\textbf{SST*} can find an initial solution very quickly; however the solution quality is initially poor, especially with higher-dimensional systems (rows 6--11).
The convergence is slow -- our \SI{5}{min} timeout was not sufficient for SST* to fully converge in any of the cases.
\textbf{Geometric RRT*+KOMO} can find near-optimal initial solutions, but does not work well in instances that require long trajectories and fails if the geometric initial guess is not close to a dynamically feasible motion.
For example, finding an initial solution in the kink and bugtrap examples (rows 7, 8) took significantly longer than parallelpark.
Another drawback is that this approach is incomplete, as visible for the v1 and v2 unicycle systems (row 4 and 5) and not globally optimal, e.g., row 10 and 11 show a very poor solution quality after \SI{5}{min}.
\textbf{kMP-db-A*} converged to the lowest-cost solution during the time limit in all cases.
At the same time, it found the highest-quality first solution in all cases, although it often took more time to compute an initial solution than the other algorithms.
We found that this is mostly caused by the challenging scenario of time-optimal planning: most motion primitives are time-optimal, i.e., result in bang-bang control.
When allowing discontinuities, the estimated time horizon is often too short for the optimizer to find a solution, requiring multiple iterations in \cref{alg:overview} to report the first solution.

For brevity, \cref{tab:results} does not include any standard deviation.
In general, we found that SBPL has almost no variance, SST* has a very high variance, and KOMO and kMP-db-A* are somewhere in between the two extremes.
One example that includes the convergence behavior as well as the variance is shown in \cref{fig:sol_ex}.

\begin{figure}
    \centering
    \includegraphics[width=0.9\linewidth]{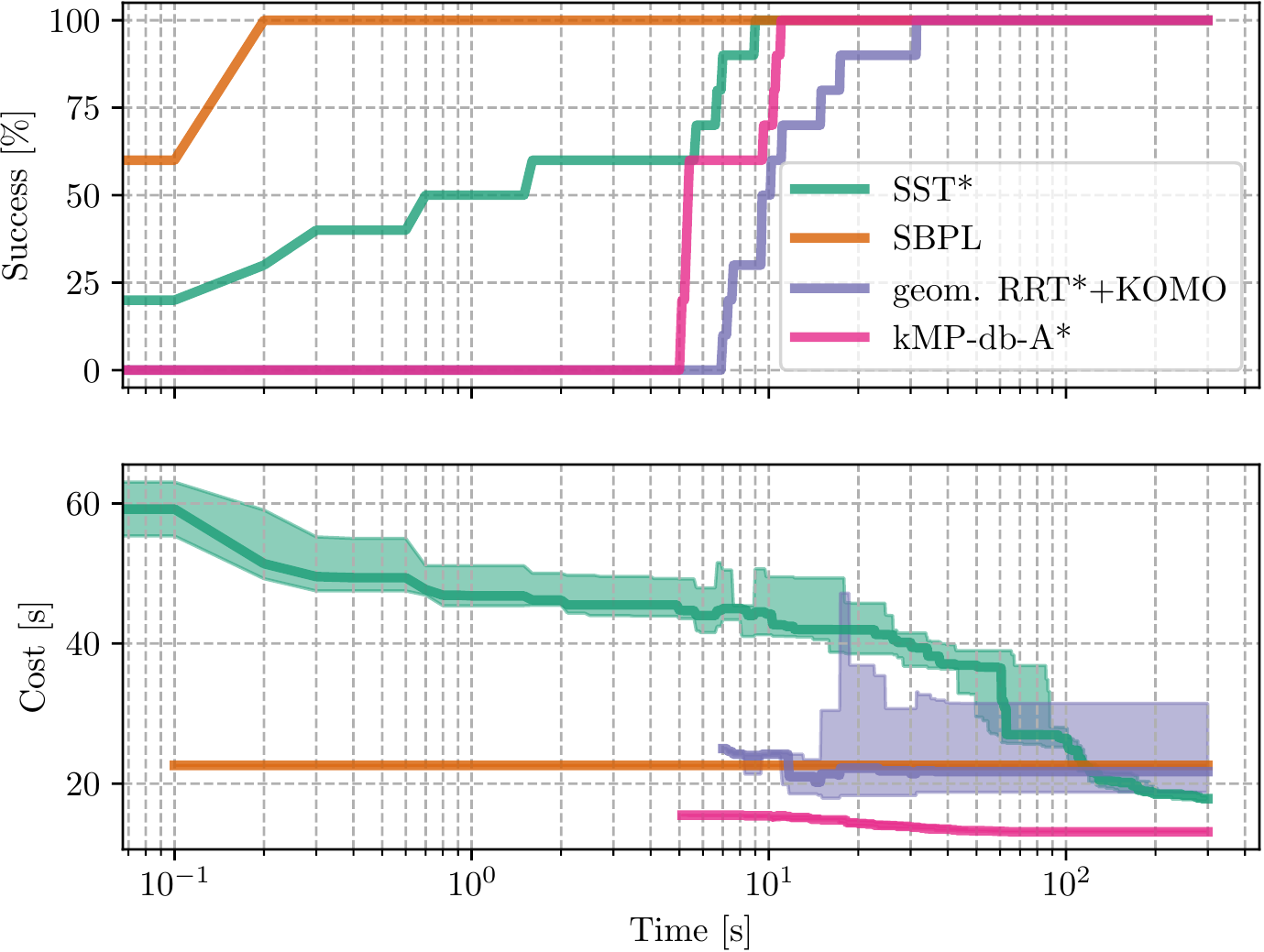}
    \caption{
    Success rate and solution cost over runtime (log-scale) for the unicycle ($1^{\text{st}}$ order) dynamical system in the kink environment (row 2 in \cref{tab:results}).
    The line is the median and the shaded region shows the first/third quartile over all trials that found a solution so far (up to 10 in total).
    }
    \label{fig:sol_ex}
\end{figure}

The runtime of the individual components of kMP-db-A* can vary widely, depending on $\delta$.
For example, in the bugtrap example for the trailer (row 11) it takes around \SI{2}{s} for db-A* to find a solution with $\delta=0.33$ and \SI{14}{s} for the optimization, while during later iterations db-A* requires \SI{46}{s} ($\delta=0.12$) and the optimization only \SI{6}{s}.

\section{Conclusion} 
\label{sec:conclusion}

We present a new kinodynamic motion planning technique, kMP-db-A*, that uses a novel graph-search method with trajectory optimization in an iterative fashion.
For the graph search, we introduce db-A*, a generalization of A* that reuses motion primitives to compute trajectories with a bounded discontinuity.
Then, we warm-start trajectory optimization using the output of the graph search and compute new motion primitives online.
KMP-db-A* combines ideas and advantages of sampling-based, search-based, and optimization-based kinodynamic motion planners: it converges asymptotically to the optimal solution, directly solves for the time horizon, finds a near-optimal solution quickly, and does not require any additional post-processing.

The major limitation of kMP-db-A* is that it sometimes requires a long time to compute an initial solution.
We believe that this is not a fundamental issue and that it can be improved using the following techniques in the future.
First, we are interested in using stronger heuristics and bounded suboptimal and incremental graph search techniques to reuse information between iterations.
Second, we plan to investigate the use of optimizers that do not operate over the full trajectory time horizon.
Finally, we believe that our work also lays the foundation for novel kinodynamic multi-robot motion planners.

\bibliographystyle{IEEEtran}
\bibliography{IEEEabrv,references}

\end{document}